\documentclass[5p]{elsarticle} 

\usepackage{lineno,hyperref}
\modulolinenumbers[5]

\journal{Physica D}

\usepackage[utf8]{inputenc} 
\usepackage[T1]{fontenc}    
\usepackage{hyperref}       
\usepackage{url}            
\usepackage{booktabs}       
\usepackage{amsfonts}       
\usepackage{nicefrac}       
\usepackage{microtype}      
\usepackage{lipsum}
\usepackage{amsmath}
\usepackage{amsfonts}
\usepackage{amsthm}
\usepackage{amssymb,color}
\usepackage{mathtools}
\newtheorem{theorem}{Theorem}
\newtheorem{lemma}[theorem]{Lemma}

\newtheorem{prop}[theorem]{Proposition}

\theoremstyle{definition}
\newtheorem{defn}[theorem]{Definition}
\theoremstyle{definition}

\numberwithin{theorem}{section}
\usepackage{graphicx}
\usepackage{subcaption}
\usepackage{mathrsfs}

\newcommand{\ba}{\begin{eqnarray}}
\newcommand{\ea}{\end{eqnarray}}
\newcommand{\nn}{\nonumber}
\newcommand{\jhpdb}{\ensuremath{b}} 

\biboptions{authoryear}

\begin{document}

\begin{frontmatter}

\title{Echo State Networks trained by Tikhonov least squares are $L^2(\mu)$ approximators of ergodic dynamical systems}

\author{Allen G. Hart}
\address{a.hart@bath.ac.uk, \\ University of Bath, UK}

\author{James L. Hook}
\address{j.l.hook@bath.ac.uk, \\ University of Bath, UK}

\author{Jonathan H. P. Dawes}
\address{j.h.p.dawes@bath.ac.uk, \\ University of Bath, UK}

\begin{abstract}
    Echo State Networks (ESNs) are a class of single-layer recurrent neural networks with randomly generated internal weights, and a single layer of tuneable outer weights, which are usually trained by regularised linear least squares regression. Remarkably, ESNs still enjoy the universal approximation property despite the training procedure being entirely linear. In this paper, we prove that an ESN trained on a sequence of observations from an ergodic dynamical system (with invariant measure $\mu$) using Tikhonov least squares regression against a set of targets, will approximate the target function in the $L^2(\mu)$ norm. In the special case that the targets are future observations, the ESN is learning the next step map, which allows time series forecasting.
    We demonstrate the theory numerically by training an ESN using Tikhonov least squares on a sequence of scalar observations of the Lorenz system.
\end{abstract}


\end{frontmatter}


\section{Introduction}

Echo state networks (ESNs) are a class of single layer recurrent neural networks introduced at the turn on the millennium independently by \cite{Jaeger2001} and \cite{doi:10.1162/089976602760407955}. These relatively simple neural networks have been used to solve a range of machine problems where the input data is a time series, including speech recognition \citep{SKOWRONSKI2007414}, learning the rules of grammar \citep{TONG2007424}, financial time series prediction \citep{TimeSeries}, \citep{LIN20097313}, short term traffic forecasting \citep{arXiv:2004.08170}, placing UAV base stations \citep{arXiv:1909.11598} and learning about the behaviour of seals \citep{arXiv:2004.08170}. ESNs are also a plausible model for the information processing of biological neurons \citep{BiologicalESN}. In this paper, we will present just enough definitions and theory to make sense of our results, but encourage the interested reader to read the recent review paper by \cite{TANAKA2019100} who cover recent developments and open questions in the field of \emph{reservoir computing}, a field of which ESN comprise a subset. The ESN is defined by the recursion relation
\begin{align}
    x_{k+1} = \sigma( Ax_k + Cz_k + \jhpdb) \nn
\end{align}
where the $x_k$ are $T$ dimensional state vectors, $\sigma : \mathbb{R}^T \to \mathbb{R}^T$ is the activation function, $A$ is the $T \times T$ \emph{reservoir matrix}, representing the connection weights between neurons, $C$ is the $T \times d$ \emph{input matrix} connecting the $d$-dimensional inputs $z_k$ to the reservoir matrix $A$, and $\jhpdb \in \mathbb{R}^T$ is a bias vector. The reservoir matrix $A$, input matrix $C$ and bias vector $\jhpdb$ are initialised randomly and remain unchanged. The ESN can be trained to approximate a sequence of target scalars $u_k$ by solving the regularised linear squares problem
\begin{align}
    \min_{W}\sum_{k=0}^{\ell-1} \lVert W^{\top} x_{k} - u_k \rVert^2 + \lambda\lVert W \rVert^2 \nonumber
\end{align}
where $\lambda > 0$ is the Tikhonov regularisation parameter.
If the target scalars $u_k$ are equal to the observations $z_k$, then the ESN is being trained to predict the future. To see this, we can set up a sequence of scalars $v_k$ defined by the recurrent relation
\begin{align}
    &v_{k+1} = W^{\top}s_k \nonumber \\
    &s_{k+1} = \sigma( As_k + Cv_{k+1} + \jhpdb) \label{informal_autonomous_phase}
\end{align}
and we then hope that $v_k \approx u_k$ for sufficiently many future values of $k$. We can view $s_k$ as the state of a discrete time autonomous dynamical system which we will call the ESN autonomous phase. In this paper, we will suppose $z_k$ are a sequence of sequential observations from an ergodic dynamical system, with invariant measure $\mu$. We will go on to prove that ESNs trained by least squares can approximate arbitrary target functions (including one that returns future observations) of the ergodic dynamical system in the $L^2(\mu)$ norm. This theorem is closely related to recent work by \cite{Verzelli} discussing the connection between ergodic dynamical systems and feasible learning.
The result also explains the remarkable success of ESNs trained on dynamical systems explored numerically by, for example, \cite{Jaeger2001}, \cite{1556081}, \cite{416607}, \cite{4118282}, \cite{5645205}, \cite{YILDIZ20121}, \cite{Pathak2017}, \cite{Lokse2017}, \cite{YEO2019671}, \cite{Ashesh_2019}, \cite{Vlac_2019}, \cite{Embedding_and_approximation_theorems}.

The remainder of the paper is organised as follows. In section \ref{Training_Thm_for_ESNs} we define an ergodic dynamical system and present Birkhoff's ergodic theorem. Next, in section \ref{training_theorem}, we introduce the major result of this paper (Theorem \ref{least_sqs_thm}), stating that an ESN trained on a sequence of observations from an ergodic dynamical system using Tikhonov least squares will $L^2(\mu)$ approximate an arbitrary target function. This arbitrary target function could be the \emph{next step map} used for forecast the future of the time series. Furthermore, we discuss the \emph{central limit theorem for ergodic dynamical systems} in connection with the number of data points that are required for a good approximation.

In section \ref{Lorenz_is_stably_mixing} we present the work of \cite{Luzzatto2005} culminating in a proof that the Lorenz attractor is stably mixing, hence ergodic - explaining the success of so many authors using an ESN to forecast the trajectory of the Lorenz system. 

In section \ref{Numerical_section} we numerically simulate a trajectory of the Lorenz system. We observed the $x$-component of the system (which we called $\xi$ to avoid notational clash) and assigned the $z$ components (which we denote $\zeta$) as targets. We explored how the approximation of the targets $\zeta_k$ given the observations $\xi_k$ improved as the number of data points $(\xi_k,\zeta_k)$ grew.
Finally, in section \ref{conclusions} we summarise the paper and discuss ideas for future work.

\section{Preliminaries on Ergodic Theory}
\label{Training_Thm_for_ESNs}

We require that the underlying dynamical system is ergodic so that minimising the mean square differences between observations and targets does not create a bias toward areas with lots of training data. The ergodicity ensures that that training data generated from a trajectory initialised at almost any point $m_0 \in M$ will represent all dynamics on $M$. To make this formal, we will introduce the definition of ergodicty and the celebrated ergodic theorem. 

\begin{defn}
    (Generic Point) Suppose $\phi : M \to M$ is a measure preserving map with respect to the measure space $(M,\Sigma,\mu)$. Then $m_0 \in M$ is called a generic point if the orbit of $m_0$ is uniformly distributed over $M$ according to the measure $\mu$.
\end{defn}

\begin{prop}
    Suppose $\phi : M \to M$ is a measure preserving map with respect to the probability space $(M,\Sigma,\mu)$ and $s \in L^1(\mu)$. Suppose $m_0$ is a generic point in $M$ then 
    \begin{align}
        \lim_{\ell \to \infty}\frac{1}{\ell} \sum_{k=0}^{\ell-1} s \circ \phi^{k}(m_0) = \int_{M} s \ d\mu. \nn
    \end{align}
\end{prop}

\begin{defn}
    (Ergodic) Let $\phi : M \to M$ be a measure preserving transformation on the probability space $(M,\Sigma,\mu)$. Then $\phi$ is ergodic if for every $\sigma \in \Sigma$ with $\phi^{-1}(\sigma) = \sigma$ either $\mu(\sigma) = 0$ or $\mu(\sigma) = 1$.
\end{defn}

\begin{theorem}
    (Ergodic Theorem \citep{Birkhoff656}) Suppose $\phi : M \to M$ is ergodic with respect to the probability space $(M,\Sigma,\mu)$ and $s \in L^1(\mu)$. Then $\mu$-almost all $m_0 \in M$ are generic hence for $\mu$-almost all $m_0 \in M$
    \begin{align}
        \lim_{\ell \to \infty}\frac{1}{\ell} \sum_{k=0}^{\ell-1} s \circ \phi^{k}(m_0) = \int_{M} s \ d\mu. \label{Ergodic_Theorem_eqn}
    \end{align}
\end{theorem}

The left hand side of \eqref{Ergodic_Theorem_eqn} is called the \emph{time average} taken from initial point point $m_0 \in M$, and the right hand side called the \emph{space average}. The ergodic theorem then states that the time average taken from almost all initial points equals the space average.

\section{A Training Theorem for Echo State Networks}
\label{training_theorem}
\subsection{Preliminaries}

Suppose we have have an ergodic dynamical system $\phi : M \to M$, and we can observe the dynamics via an observation map $g : M \to \mathbb{R}^T$ and target map $u : M \to \mathbb{R}$. A trajectory originating from a generic point $m_0 \in M$ will ergodically explore the space $M$ and yield a sequence of observations $g \circ \phi^{k}(m_0)$ and targets $u \circ \phi^{k}(m_0)$ for $k = 0, 1 , 2 , ... , \ell$. 

Suppose we compute the vectors $W_{\ell} \in \mathbb{R}^T$ minimising the regularised least squares difference between the mapping of the observations $W^{\top} g \circ \phi^k(m_0)$ and the targets $u \circ \phi^k(m_0)$. We prove in the next lemma that as the number of data points $\ell$ grows large, the least squares solution $W_{\ell}$ minimises the ergodic average difference between the mapping on the observations $W^{\top} g \circ \phi^k(m_0)$ and the targets $u \circ \phi^k(m_0)$.

\begin{lemma}
    Let $(M,\Sigma)$ be a measurable space, and suppose that $\phi : M \to M$ is ergodic with invariant measure $\mu$. Let $m_0$ be a generic point in $M$. Let $g \in L^2(\mu)(M,\mathbb{R}^T)$ be an observation function and suppose that $u \in L^2(\mu)(M, \mathbb{R})$ is a target function we wish to approximate. 
        
    Let $\lambda > 0$. Define the sequence $(W_{\ell})_{\ell \in \mathbb{N}}$ such that, for each $\ell \in \mathbb{N}$, the vector $W_{\ell} \in \mathbb{R}^T$ is the unique minimiser of the regularised least squares difference
    \begin{align*}
        \frac{1}{\ell} \bigg( \sum_{k = 0}^{\ell-1} \rVert W^\top g \circ \phi^k(m_0) - u \circ \phi^k(m_0) \lVert^2 + \lambda \lVert W \rVert^2 \bigg).
    \end{align*}
    Then, the sequence $(W_{\ell})_{\ell \in \mathbb{N}}$ converges to
    \begin{align*}
        W_{\infty} &= \bigg( \int_M g(m)g(m)^{\top} \ d \mu(m) + \lambda I \bigg)^{-1} \\
        &\times \int_M u(m)g(m) \ d \mu(m)
    \end{align*}
    which is the unique minimiser of
    \begin{align*}
            \lVert W^{\top} g - u \rVert_{L^2(\mu)}^2 + \lambda \lVert W \rVert^2.
    \end{align*}
    \label{W_ell_lemma}
\end{lemma}

\begin{proof}
    Consider the map $\Psi : \mathbb{R}^T \to \mathbb{R}$ defined
    \begin{align*}
        \Psi(W) &= \lVert W^{\top} g - u \rVert_{L^2(\mu)}^2 + \lambda \lVert W \rVert^2 \\
        &= \int_M \lVert W^{\top} g(m) - u(m) \rVert^2 \ d \mu(m) + \lambda \lVert W \rVert^2.
    \end{align*}
    The minimiser of $\Psi$ satisfies $D \Psi = 0$ where $D$ is the derivative operator, so we consider
    \begin{align*}
        0 &= (D\Psi)(W) \\
        &= D \bigg( \int_M \lVert W^{\top} g(m) - u(m) \rVert^2 \ d \mu(m) + \lambda \lVert W \rVert^2 \bigg) \\
        &= \int_M D \lVert W^{\top} g(m) - u(m) \rVert^2 \ d \mu(m) + \lambda D \lVert W \rVert^2 \\
        &= \int_M 2( W^{\top} g(m) - u(m) )g(m)^{\top} \ d \mu(m) + 2 \lambda W^{\top} \\
        &= \int_M ( W^{\top} g(m) - u(m) )g(m)^{\top} \ d \mu(m) + \lambda W^{\top} \\
        &= W^{\top} \int_M g(m)g(m)^{\top} \ d \mu(m) - \int_M u(m)g(m)^{\top} \ d \mu(m) \\
        &+ \lambda W^{\top} I \\
        &= W^{\top}\bigg( \int_M g(m)g(m)^{\top} \ d \mu(m) + \lambda I \bigg) \\
        &- \int_M u(m)g(m)^{\top} \ d \mu(m),
    \end{align*}
    which upon rearrangement yields 
    \begin{align*}
        W &= \bigg( \int_M g(m)g(m)^{\top} \ d \mu(m) + \lambda I \bigg)^{-1} \\
        &\times \int_M u(m)g(m) \ d \mu(m).
    \end{align*}
    Since this is the unique solution to $0 = D\Psi(W)$, this stationary point is unique, and we will denote it $W_{\infty}$. We can see it is a minimum because the Hessian $H\Psi$ is positive definite. Next, define the map
    \begin{align*}
        \Phi : \{ y \in C^1(\mathbb{R}^T,\mathbb{R}) \ | \ y \text{ has a unique minimum} \} \to \mathbb{R}^T
    \end{align*}
    as the mapping on the $C^1$ functions with a unique minumum that returns their unique minimum. We can see that $\Phi$ is continuous with respect to the $C^1$ topology and standard topology on $\mathbb{R}$ respectively. We consider the family of functions $y_{\ell} \in \{ y \in C^1(\mathbb{R}^T, \mathbb{R})$ \ | \ $y$ has a unique minimum\}
    \begin{align*}
        y_{\ell}(W) =
        \frac{1}{\ell} \bigg( \sum_{k = 0}^{\ell-1} \rVert W^\top g \circ \phi^k(m_0) - u \circ \phi^k(m_0) \lVert^2 + \lambda \lVert W \rVert^2 \bigg),
    \end{align*}
    so that by definition
    $W_{\ell} = \Phi(y_{\ell}(W))$ and hence
        \begin{align*}
        \lim_{\ell \to \infty} W_{\ell} &= \lim_{\ell \to \infty} \Phi(y_{\ell}(W)) \\
        &= \Phi\bigg(\lim_{\ell \to \infty} y_{\ell}(W) \bigg)
        &= \Phi\bigg( \lVert W^{\top} g - u \rVert_{L^2(\mu)}^2 + \lambda \lVert W \rVert^2 \bigg) \\
        &= W_{\infty}.
    \end{align*}
where we have used, respectively, continuity of $\Phi$ and the Ergodic Theorem.
\end{proof}

\subsection{Echo State Networks}

An Echo State Network is a special case of a more general system called a state space system, or reservoir system. These are maps of the form $F : \mathbb{R}^N \times \mathbb{R}^d \to \mathbb{R}^N$, which admit an ESN as a special case when
\begin{align*}
    F(x,z) = \sigma(Ax + Cz + \jhpdb).
\end{align*}
If a state space system is contracting in the state variable, i.e there exists a $c \in [0,1)$ such that
    \begin{align*}
        \rVert F(x,z) - F(y,z) \lVert \leq c\lVert x - y \rVert,
    \end{align*}
    and the inputs $u_k$ are the observations of a dynamical system i.e $u_k = \omega \circ \phi^k(m_0)$ then there is a continuous map $f \in C^0(M,\mathbb{R}^N)$ synchronising the dynamics of $\phi$ on $M$ to the dynamics of the reservoir states $x_k$. The map $f$ is called a state synchronisation map (SSM) and is a generalised synchronisation in the sense described by \cite{PhysRevLett.76.1816}. We can guarantee that an ESN is state contracting by bounding the 2-norm of the reservoir matrix $\lVert A \rVert_2 < 1$. An important existence result for SSMs is the following theorem, due to
    \cite{chaos_on_compacta}.
\begin{theorem}
\label{SSM_thm}
    \citep{chaos_on_compacta} Let $M$ be a topological space, $\phi \in \text{Hom}(M)$ be a dynamical system, and $\omega \in C^0(M,\mathbb{R}^d)$ an observation function. Suppose that the state space system $F : \mathbb{R}^N \times \mathbb{R}^d \to \mathbb{R}^N$ is state contracting, i.e there exists a $c \in [0,1)$ such that
    \begin{align*}
        \rVert F(x,u) - F(y,u) \lVert \leq c\lVert x - y \rVert.
    \end{align*}
    Then there exists a unique $f \in C^0(M,\mathbb{R}^N)$ called the state synchronisation map (SSM) such that, for any $m_0 \in M$ and $x_0 \in \mathbb{R}^N$ the sequence 
    \begin{align*}
        x_{k+1} = F(x_k,\omega\circ\phi^k(m_0))
    \end{align*}
    originating at $x_0$ converges to $f\circ\phi^k(m_0)$ as $k \to \infty$.
\end{theorem}

In order to approximate the arbitrary dynamics of $\phi$ via the observation function $\omega$ using state space systems, we require that
the state space maps $F$ possess some sort of universal approximation property. Thus, we will define a class of \emph{linear universal approximators} with respect to an arbitrary complete norm $\lVert \cdot \rVert$. Every class of linear universal approximators contains maps, which after composition with another suitable map, forms a state map.

\begin{defn}
    Let $\mathcal{F}$ be a sequence of maps $\{F_T\} : \mathbb{R}^N \times \mathbb{R}^d \to \mathbb{R}^T$. Let $C \subset \mathbb{R}^N$ and $K \subset \mathbb{R}^d$ be vectors and let $\Omega(C \times K, \mathbb{R})$ be a Banach space of real valued functions on $C \times K$, with norm denoted $\lVert \cdot \rVert_{\Omega}$. If, for any $g \in \Omega(C \times K,\mathbb{R})$ and any $\epsilon > 0$ there exists an $T_0 \in \mathbb{N}$ such that for any $T > T_0$ there exists a $W_* \in \mathbb{R}^T$ such that
    \begin{align*}
        \rVert W_*^{\top}F_T - g \lVert_{\Omega} < \epsilon
    \end{align*}
    then we say that $\mathcal{F}$ is a class of \emph{linear universal approximators} on $\Omega(C \times K, \mathbb{R}^N \times \mathbb{R}^d)$. 
\end{defn}

A widely used class of linear universal approximators is the class of Echo State Networks with randomly initialised internal weights, as shown by the following result.

\begin{theorem}
    Let $\mathcal{F}$ denote the sequence of maps $\{ F_T \} : \mathbb{R}^N \times \mathbb{R}^d \to \mathbb{R}^T$ defined by
    \begin{align*}
        F_T(x,z) = \sigma(Ax + Cz + \jhpdb)
    \end{align*}
    where
    \begin{itemize}
        \item $\sigma \in C^{1}(\mathbb{R})$ is $1$-finite (see \cite{HORNIK1990551} for the definition of $\ell$-finite)
        \item $A$ is a $T \times N$ random matrix, where $T >N$ and the first $N$ rows of $A$ form an $N \times N$ random submatrix with 2-norm less than 1 almost surely. The $j^{\mathrm{th}}$ row of $A$ (where $j > N$), denoted $A_j$, is a random variable with full support on $(\mathbb{R}^N)^{\top}$
        \item $C$ is a $T \times d$ random matrix with $j^\mathrm{th}$ row $C_j$, a random variable with full support on $(\mathbb{R}^d)^{\top}$
        \item $\jhpdb$ is a random $T$-vector with $j\mathrm{th}$ entry $\jhpdb_j$, a random variable with full support on $\mathbb{R}$.
    \end{itemize}
    Let $C \times K$ be an arbitrary compact subset of $\mathbb{R}^N \times \mathbb{R}^d$. Then, almost surely, $\mathcal{F}$ is  a class of linear universal approximators on $L^2(C \times K, \mathbb{R})$.
    \label{ESNs_are_universal}
\end{theorem}

\begin{proof}
    Fix $g \in L^2(C \times K, \mathbb{R})$ and $\epsilon > 0$. Then for any $\alpha \in (0,1)$, it follows from the Random Universal Approximation Theorem \citep[Theorem 2.4.5.]{Embedding_and_approximation_theorems}) that there exists a $T_0 \in \mathbb{N}$ such that for any $T > T_0$, with probability at least $\alpha$,
    \begin{align*}
        \rVert W^\top F_T - g \lVert_{L^2} < \epsilon,
    \end{align*}
    hence $\mathcal{F}$ is a class of linear universal approximators. Since $\mathcal{F}$ is a class of linear universal approximators for any $\alpha \in (0,1)$, $\mathcal{F}$ is almost surely a class of linear universal approximators.
\end{proof}

To construct such an ESN in practice, we create a reservoir system $F : \mathbb{R}^T \times \mathbb{R}^d \to \mathbb{R}^T$ by defining
\begin{align*}
    F(x,z) = \sigma \big( [A,0]x + Cz + \jhpdb \big)
\end{align*}
where $[A,0]$ is the $T \times T$ matrix where the first $N$ columns form the matrix $A$ and the remaining columns are 0. Suppose we truncate at $N$ the state vectors $x \in \mathbb{R}^T$ by applying the canonical projection $\pi : \mathbb{R}^T \to \mathbb{R}^N$, and denote the truncation $\pi(x) = \bar{x} \in \mathbb{R}^N$. The dynamics of the truncated vectors $\bar{x}$ are given by the (state contracting) state space system $\pi \circ F_T : \mathbb{R}^N \times \mathbb{R}^d \to \mathbb{R}^N$, which is also an ESN as is defined by
\begin{align*}
    \pi \circ F_T(\bar{x},z) = \sigma(\bar{A}\bar{x} + \bar{C}z + \bar{\jhpdb}).
\end{align*}
Here, the $N \times N$ reservoir matrix $\bar{A}$ is created by truncating at $N$ the rows and columns of $A$. The $N \times d$ input matrix $\bar{C}$ is created by truncating at $N$ the rows of $C$. The $N$-vector $\bar{\jhpdb}$ is created by truncating at $N$ the entries of $\jhpdb$.
We conclude that Echo State Networks with (appropriately chosen) randomly generated internal weights are a class of linear universal approximators that each give rise to a state synchronisation map.

We demanded that the $T \times T$ reservoir matrix take the form $[A,0]$, whereas in practice, the reservoir matrix does not have this structure. We imposed this condition to simplify the proofs, but we believe, based on numerical evidence in the literature, that this choice of shape is not necessary. 

There is one more technical lemma we will include here before presenting the main theorem (Theorem \ref{least_sqs_thm}) of the paper. Recall that topological spaces have a natural Borel sigma algebra and are therefore measurable spaces. On such spaces we can integrate real valued functions. If $A$ and $B$ are homeomorphic topological spaces, then integration on $A$ is essentially the same as integration on $B$. We use this observation in Theorem \ref{least_sqs_thm} to move between integration on the topological space $M$ to integration on the image $f(M)$. This demands the highly non-trivial assumption that the SSM $f$ is a homeomorphism. The observation is made formal in the following lemma.

\begin{lemma}
    (Change of variables) Let $A,B$ be homeomorphic topological spaces and suppose $y \in \text{Hom}(A,B)$. The topologies on $A,B$ induce Borel Sigma algebras $\mathscr{A},\mathscr{B}$ on $A,B$ respectively. Let $\mu_A$ be a measure on $A$ and $\mu_B$ a measure on $B$ (called the pushforward measure) defined $\mu_B(b) = \mu_A(y^{-1}(b))$ for all $b \in \mathscr{B}$. Then for any $\mu_B$ measurable function $g : B \to \mathbb{R}$
    \begin{align*}
        \int_A g \circ y \ d \mu_A = \int_B g \ d \mu_B.
    \end{align*}
    \label{gy_lemma}
\end{lemma}

\begin{proof}
    This is a special case of Theorem 3.6.1 in \cite{Bogachev}.
\end{proof}

\subsection{A Training Theorem For ESNs}

Before we finally plunge into the statement and proof of the main theorem, we will describe the result in words.
Suppose we have an ergodic dynamical system $\phi : M \to M$, which we observe via the function $\omega : M \to \mathbb{R}^d$ and that our goal is to approximate a target function $u : M \to \mathbb{R}$.
Suppose we have at our disposal a class $\mathcal{F}$ of linear universal approximating state maps. For example, $\mathcal{F}$ could be a collection of arbitrarily high dimensional ESNs. Make the additional (and non trivial) assumption that the state maps give rise to an SSM that is homeomorphic onto its image. Suppose then that the state map $F$ is driven with observations of a trajectory $z_k = \omega \circ \phi^k(m_0)$ originating from a generic point $m_0$. This creates a sequence of reservoir states $x_k$ that satisfy
\begin{align*}
    x_{k+1} = F(x_k,z_k).
\end{align*}
We also assemble a sequence of scalar targets $u \circ \phi^k(m_0)$.

Suppose we use regularised least squares regression to minimise the difference between the linear mapping on the observations $W^{\top} x_k$ and the targets $u \circ \phi^k(m_0)$. Then we can conclude that the ergodic average difference between the mapping on the data and the target map $u$ can be made smaller than the arbitrary threshold $\epsilon$. This requires that the trajectory length $\ell$ and state map dimension $T$ are sufficiently large, while ensuring the regularisation parameter $\lambda > 0$ is sufficiently small.

We remark that a notable weakness of Theorem \ref{least_sqs_thm} is its non-constructive natue, because the actual values for $\ell$, $T$ and $\lambda$ are not computed in terms of $\epsilon$.

\begin{theorem}
    Let $M$ be a topological space, and suppose that $\phi\in \text{Hom}(M)$ is ergodic with invariant measure $\mu$. Let $m_0$ be a generic point in $M$. Let $\omega \in C^0(M,\mathbb{R}^d)$ be the observation function and suppose that $u \in L^2(\mu)(M, \mathbb{R})$ is a target function we wish to approximate. 
    
    Suppose that $\mathcal{F}$ is a class of linear universal approximators on $L^2(C \times K, \mathbb{R})$ on every compact $C \subset \mathbb{R}^N,K \subset \mathbb{R}^d$. Let $(s_T)_{T \in \mathbb{N}} : \mathbb{R}^T \to \mathbb{R}^N$ be a sequence of maps. Suppose (for each large enough $T$) the state map $s_T \circ F_T : \mathbb{R}^N \times \mathbb{R}^d \to \mathbb{R}^N$ admits an SSM $f \in \text{Hom}(M,f(M))$.
    For each $T, \ell \in \mathbb{N}$, and $\lambda > 0$ let $W_{\ell} \in \mathbb{R}^T$ be the vector obtained by minimising the regularised least squares difference
    \begin{align*}
        \sum_{k = 0}^{\ell} \rVert W^\top F_T(f \circ \phi^{k-1}(m_0) , \omega \circ \phi^k(m_0)) - u \circ \phi^k(m_0) \lVert^2& \\
        + \lambda \lVert W \rVert^2&.
    \end{align*}
    Then, for any $\epsilon > 0$, there exists $\lambda^* > 0$ and $\ell_0, T_0 \in \mathbb{N}$ such that for all $\lambda \in (0,\lambda^*)$ and $\ell > \ell_0, T > T_0 $
    \begin{align*}
        \lVert W_{\ell}^{\top} F_T(f \circ \phi^{-1}, \omega) - u \rVert_{L^2(\mu)}^2 < \epsilon.
    \end{align*}
    \label{least_sqs_thm}
\end{theorem}

\begin{proof}
    Let $y : M \to y(M) \subset (\mathbb{R}^N \times \mathbb{R}^d)$ be defined by
    \begin{align*}
        y(m) = ( f \circ \phi^{-1}(m) , \omega(m) ) \ \forall m \in M
    \end{align*}
    and note that $F_T(f \circ \phi^{-1}, \omega) = F_T \circ y$ and that
    $y \in \text{Hom}(M, y(M) )$ because $f \in \text{Hom}(M , f(M))$. Now fix $\epsilon > 0$. Let $\mu'$ be a measure defined on $y(M) \subset (\mathbb{R}^N \times \mathbb{R}^d)$ by $\mu'(\sigma) = \mu(y^{-1}(\sigma))$ for all measurable subsets $\sigma$ of $f(M)$.
    Using the assumption that $\mathcal{F}$ is a class of linear universal approximators, we can choose $T_0$ sufficiently large that for any $T > T_0$ there exists $W_* \in \mathbb{R}^T$ such that
    \begin{align*}
        \lVert W_*^{\top} F_T - u \circ y^{-1} \rVert_{L^2(\mu')}^2 < \frac{\epsilon}{3},
    \end{align*}
    hence (by lemma \ref{gy_lemma})
    \begin{align*}
        \lVert W_*^{\top} F_T \circ y - u \rVert_{L^2(\mu)}^2 = \lVert W_*^{\top} F_T - u \circ y^{-1} \rVert_{L^2(\mu')}^2  < \frac{\epsilon}{3}.
    \end{align*}
    Now let
    \begin{align*}
        \lambda^* = \frac{\epsilon}{3 \lVert W_* \rVert^2}
    \end{align*}
    and $\lambda \in (0,\lambda^*)$. Define the sequence $(W_{\ell})_{\ell \in \mathbb{N}}$ such that, for each $\ell \in \mathbb{N}$, the vector $W_{\ell} \in \mathbb{R}^T$ is the unique minimiser of the regularised least squares difference
    \begin{align*}
        \frac{1}{\ell} \bigg( \sum_{k = 0}^{\ell-1} \rVert W^\top F_T(f \circ \phi^{k-1}(m_0) , \omega \circ \phi^k(m_0)) - u \circ \phi^k(m_0) \lVert^2& \\ 
        + \lambda \lVert W \rVert^2& \bigg).
    \end{align*}
    By lemma \ref{W_ell_lemma}, $(W_{\ell})_{\ell \in \mathbb{N}}$ converges as $\ell \to \infty$ to $W_{\infty}$ which minimises 
    \begin{align*}
        \lVert W^{\top} F_T(f \circ \phi^{-1}, \omega) - u \rVert_{L^2(\mu)}^2 + \lambda \lVert W \rVert^2.
    \end{align*}
    Now we choose $\ell_0$ such that for all $\ell > \ell_0$ 
    \begin{align*}
        \rVert W_{\ell}^T F_T(f \circ \phi^{-1}, \omega) - W_{\infty}^T F_T(f \circ \phi^{-1}, \omega) \lVert_{L^2(\mu)}^2 < \frac{\epsilon}{3}.
    \end{align*}
    Now the proof proceeds directly
    \begin{align*}
        &\lVert W_{\ell}^T F_T(f \circ \phi^{-1},\omega) - u \rVert_{L^2(\mu)}^2 \\
        =&\lVert W_{\ell}^T F_T(f \circ \phi^{-1}, \omega) - W_{\infty}^T F_T(f \circ \phi^{-1}, \omega) \\
        +& W_{\infty}^T F_T(f \circ \phi^{-1},\omega) - u \rVert_{L^2(\mu)}^2 \\
        \leq&\lVert W_{\ell}^T F_T(f \circ \phi^{-1}, \omega) - W_{\infty}^T F_T(f \circ \phi^{-1}, \omega) \rVert_{L^2(\mu)}^2 \\
        +& \lVert W_{\infty}^T F_T(f \circ \phi^{-1}, \omega) - u \rVert_{L^2(\mu)}^2 \\
        <& \frac{\epsilon}{3} + \lVert W_{\infty}^T F_T(f \circ \phi^{-1}, \omega) - u \rVert_{L^2(\mu)}^2 \\
        \leq&  \frac{\epsilon}{3} + \lVert W_{\infty}^T F_T(f \circ \phi^{-1}, \omega) - u \rVert_{L^2(\mu)}^2 + \lambda \lVert W_{\infty} \rVert^2 \\
        \leq& \frac{\epsilon}{3} + \lVert W_{*}^T F_T(f \circ \phi^{-1}, \omega) - u \rVert_{L^2(\mu)}^2 + \lambda \lVert W_{*} \rVert^2 \\
        <& \frac{\epsilon}{3} + \frac{\epsilon}{3} + \lVert W_{*}^T F_T(f \circ \phi^{-1}, \omega) - u \rVert_{L^2(\mu)}^2 \\
        =& \frac{\epsilon}{3} + \frac{\epsilon}{3} + \lVert W_{*}^T F_T \circ y - u \rVert_{L^2(\mu)}^2 \\
        <& \frac{\epsilon}{3} + \frac{\epsilon}{3} + \frac{\epsilon}{3} = \epsilon.
    \end{align*}
    \end{proof}

Theorem \ref{least_sqs_thm} guarantees an approximation in the $L^2(\mu)$ norm, which is sadly weaker than the $C^1$ norm. That is to say, a sequence which converges in $C^1$ also converges in $L^2(\mu)$, but the converse does not hold in general. This distinction is particularly relevant when the problem is chaotic time series forecasting. In this case, the target function is the \emph{next step map} $u = \omega \circ \phi$, and we recursively feed predictions into the state space map to create a trajectory into the future. An example is the ESN autonomous phase (equation \eqref{informal_autonomous_phase}). A weakness of using ESN autonomous dynamics for time series forecasting is that small approximation errors accumulate resulting in a predicted trajectory that diverges from the true trajectory in the far future. That said, \cite{Embedding_and_approximation_theorems} show that under certain conditions (crucially that the next step map $u = \omega \circ \phi$ is well approximated in the $C^1$ norm) the ESN autonomous phase will adopt dynamics that are topologically conjugate to the original dynamical system.

We must conclude that least squares regression does not guarantee a topologically conjugate autonomous phase, but we note that real data sets are contaminated by noise and finite precision arithmetic where an $L^2(\mu)$ approximation may be most suitable. Moreover, computing the (regularised) least squares solution using the SVD decomposition, or some other algorithm, is much faster than minimising the maximal pointwise distance, which may be necessary to yield a good $C^1$ approximation. Indeed, despite the theoretical limitations of the regularised least squares approach it seems to work well in practice. In fact we can interpret bad $C^1$ approximations in the parlance of machine learning as overfitted solutions, as they fit the training data well, in exactly the terms that we define a good fit, but may fail to make good predictions about the unseen future.

\subsection{Convergence rate of the time average to the space average}

Theorem \ref{least_sqs_thm} guarantees, under appropriate conditions, that with sufficiently many neurons $T$ and a sufficiently many training data $\ell$ we can obtain an arbitrarily good $L^2(\mu)$ approximation of a target function $u$. It is natural to wonder how many training data is required to achieve a given $L^2(\mu)$ approximation. To answer this, we turn our attention to the convergence rate of the time average to the space average
    \begin{align}
        \lim_{\ell \to \infty}\frac{1}{\ell} \sum_{k=0}^{\ell-1} s \circ \phi^{k}(m_0) = \int_{M} s \ d\mu \tag{\ref{Ergodic_Theorem_eqn}}
    \end{align}
as the timespan over which training data is collected grows. 
We want a uniform estimate for the rate of convergence for $s$ over all ergodic maps $\phi$. Unfortunately, no such estimate can possibly exist.
\cite{Kachurovskii1996} presents negative results that (in the author's words) \emph{leave no hope that estimates of the rate of convergence depending only on the averaged function $s$ can be obtained in ergodic theorems}. The negative results presented by \cite{Kachurovskii1996} prove that the amount of training data required is strictly dependant on the dynamical system.

Though we cannot say exactly how many data points we need for a good $L^2(\mu)$ approximation, the central limit theorem for ergodic dynamical systems suggests that for an initial point chosen uniformly over the invariant measure of $\phi$, the difference between the finite time average and space average converges to a mean $0$ normal distribution with standard deviation $1 / \sqrt{\ell}$. This is made precise by the central limit theorem for ergodic dynamical systems. Before we state the theorem, we recall the definition of H{\"o}lder continuity.

\begin{defn}
    (H{\"o}lder continuous) Let $(M,d)$ be a metric space. A map $s : M \to \mathbb{R}$ is called \emph{H{\"o}lder continuous} if there exist constants $p \in (0,1]$ and $K > 0$ such that
    \begin{align*}
        \lVert s(m) - s(m') \rVert \leq K d(m,m')^p
    \end{align*}
    for all $m,m' \in M$.
\end{defn}

\begin{theorem}
    (Central limit theorem for ergodic dynamical systems) Let $\phi : M \to M$ be ergodic with respect to the probability space $(M,\Sigma,\mu)$. Let $X_0$ be a uniform random variable with respect to the space $(M,\Sigma,\mu)$. Let $s \in L^1(\mu)(M,\mathbb{R})$ be H{\"o}lder continuous and denote the
    space average of $s$ by
    \begin{align*}
        \mathbb{E}[s] := \int_{M} s \ d\mu.
    \end{align*}
     Let the random variables $X_j := s \circ \phi^j(X_0)$ for $j=0,\ldots,\ell-1$ and denote the partial sum
     $S_\ell=X_0 + \cdots + X_{\ell-1}$.
     Then, for some $\sigma > 0$, the partial sum
     $S_\ell$ satisfies the central limit theorem:
    \begin{align}
        \lim_{\ell \to \infty} \mu \bigg( \bigg\{ \frac{S_\ell - \ell \mathbb{E}[s]}{\sqrt{\ell}} \leq z \bigg\} \bigg) = \frac{1}{2 \pi \sigma} \int_{-\infty}^z \mathrm{e}^{-\frac{\tau^2}{2 \sigma^2}} \ d\tau \nn
    \end{align}
    almost surely, or in other words $(S_\ell - \ell \mathbb{E}[s])/\sqrt{\ell}$ converges in law to $\mathcal{N}(0,\sigma^2)$.
\end{theorem}

\begin{proof}
    \cite{ErgodicTheory2010}.
\end{proof}

To see the connection between the central limit theorem and the work in this paper, suppose we choose a map $s$ that returns the matrix vector pair 
\begin{align*}
    s(m_0) &= \bigg( \big[f(m_0)f^\top(m_0) + I \lambda \big], f(m_0) u(m_0) \bigg) \\
    &=: (\Sigma_0, v_0),
\end{align*}
and define a sequence of pairs with $\ell$th pair
\begin{align*}
    (\Sigma_\ell, v_\ell) := \frac{1}{\ell}\sum_{k=0}^{\ell-1} s \circ \phi^k(m_0).
\end{align*}
Then it follows that
\begin{align*}
    W_\ell = \Sigma_\ell^{-1} v_\ell
\end{align*}
is the linear readout layer obtained by regularised least squares regression using $\ell$ data points. Furthermore, it follows from the central limit theorem that
for random initial points $m_0$ (distributed uniformly with respect to the invariant measure $\mu$) the sequence $(\Sigma_\ell, v_\ell)_{\ell \in \mathbb{N}}$ converges in law to a (multivariate) normal distribution, with variance converging with order $1/\ell$, and mean $(\Sigma,v)$ which satisfies 
\begin{align*}
    W_{\infty} = \Sigma^{-1}v.
\end{align*}
We note that the convergence of $(\Sigma_\ell, v_\ell)_{\ell \in \mathbb{N}}$ to $(\Sigma, v)$ with order $1/\sqrt{\ell}$ does not necessarily imply that $(W_{\ell})_{\ell \in \mathbb{N}}$ converges to $W_{\infty}$ at the same rate.

\section{The Lorenz attractor is stably mixing}
\label{Lorenz_is_stably_mixing}

We have shown that we can approximate, in the $L^2(\mu)$ sense, any target function on an ergodic dynamical system using an ESN and Tikhonov least squares. This partially explains the success enjoyed by \cite{Jaeger2001}, \cite{1556081}, \cite{416607}, \cite{4118282}, \cite{5645205},  \cite{Pathak2017}, \cite{Lokse2017}, \cite{YEO2019671}, \cite{Ashesh_2019}, \cite{Vlac_2019}, and \cite{Embedding_and_approximation_theorems}. Many authors including \cite{Ashesh_2019} successfully predict the future observations of the Lorenz system, while \cite{Pathak2017}, \cite{Vlac_2019}, and \cite{Embedding_and_approximation_theorems} additionally recover topological invariants including Lyapunov exponents, fixed point eigenvalues and homology groups. The authors are successful in their numerical experiments because the Lorenz attractor is \emph{mixing} which implies it is ergodic, suggesting the conditions Theorem \ref{least_sqs_thm} hold and we can $L^2(\mu)$ approximate target functions on the Lorenz attractor.

Proving that the Lorenz attractor is mixing was a tremendous achievement, built upon the works of \cite{1977DoSSR.234..336A}, \cite{Structural_stability_of_Lorenz_attractors}, \cite{pesin_1992}, \cite{PMIHES_1979__50__73_0}, and \cite{TUCKER19991197} culminating with the seminal paper by \cite{Tucker2002}, which resolved Smale's $14^\mathrm{th}$ problem 
\emph{`Is the dynamics of the ordinary differential equations of Lorenz (1963) that of
the geometric Lorenz attractor of Williams, Guckenheimer and Yorke? '} \citep{Smale1998}.
To formalise some of these ideas, we will begin with the definition of a mixing dynamical system.

\begin{defn}
    (Mixing) Let $\phi : M \to M$ be a measure preserving transformation on the measure space $(M,\Sigma,\mu)$ with $\mu(M) = 1$. Then $\phi$ is mixing if for any $A,B \in \Sigma$
    \begin{align}
        \lim_{\ell \to \infty} \mu\big(A \cap \phi^{-\ell}(B)\big) = \mu(A)\mu(B). \nn
    \end{align}
\end{defn}

\begin{lemma}
    (Mixing implies ergodic) Let $\phi : M \to M$ be a measure preserving transformation on the measure space $(M,\Sigma,\mu)$ with $\mu(M) = 1$. Suppose $\phi$ is mixing, then $\phi$ is ergodic. \label{mixing_implies_ergodic}
\end{lemma}
\begin{proof}
    Suppose $\phi$ is mixing and $A,B \in \Sigma$. Then
    \begin{align}
        \lim_{\ell \to \infty} \mu\big(A \cap \phi^{-\ell}(B)\big) &= \mu(A)\mu(B) \nn \\
        \implies \lim_{\ell \to \infty} \frac{1}{\ell} \sum_{k = 0}^{\ell-1} \mu\big(A \cap \phi^{-k}(B)\big) &= \mu(A)\mu(B) \nn \\
        \implies \lim_{\ell \to \infty} \frac{1}{\ell} \sum_{k = 0}^{\ell-1} \mu\big(A \cap \phi^{-k}(A)\big) &= \mu(A)^2 \label{last_line}.
    \end{align}
    Now suppose $\mu(A) = \mu\big(\phi^{-1}(A) \big)$. Then \eqref{last_line} reduces to $\mu(A) = \mu(A)^2$ hence $\mu(A) = 1$ or $\mu(A) = 0$, so $\phi$ is ergodic.
\end{proof}

\begin{defn}
    (Stably mixing)  Let $\phi : M \to M$ be a measure preserving transformation on the measure space $(M,\Sigma,\mu)$ with $\mu(M) = 1$. Then $\phi$ is stably mixing if sufficiently small $C^1$ perturbations of $\phi$ are mixing.
\end{defn}

\begin{theorem}
    The \cite{doi:10.1175/1520-0469(1963)020<0130:DNF>2.0.CO;2} system
\begin{equation}
\begin{array}{ll}
\dot{\xi}      & = \sigma(\upsilon - \xi) \\
\dot{\upsilon} & = \xi(\rho - \zeta) - \upsilon \\
\dot{\zeta}    & = \xi\upsilon - \beta \zeta
\end{array} \label{eqn:lorenz}
\end{equation}
with parameters $\sigma = 10$, $\beta = 8/3$, $\rho = 28$ admits a robust attractor that is stably mixing.
\label{thm:lorenz}
\end{theorem}

\begin{proof}
    \cite{Luzzatto2005}
\end{proof}

Since the Lorenz attractor is stably mixing, so is any sufficiently good $C^1$ approximation to the evolution operator $\phi$, obtained by numerical methods.
Consequently, a numerically approximated Lorenz system is ergodic, by Lemma \ref{mixing_implies_ergodic}.
Thus, we expect that an ESN, trained using Tikhonov least squares, on a sequence of observations of a numerically integrated trajectory of the Lorenz attractor will $L^2(\mu)$ approximate arbitrary target functions on the attractor.

\section{Numerical experiments}
\label{Numerical_section}

Our goal is to use an ESN to learn a mapping from the $\xi$ component of the Lorenz attractor to the $\zeta$ component. We will sample data from a single trajectory of the Lorenz attractor. To this end, let $\phi:\mathbb{R}^3 \to \mathbb{R}^3$ denote a discretisation of the Lorenz system~(\ref{eqn:lorenz}) with time step $\tau$ i.e effectively
a discrete-time map of the form
\begin{align*}
    \phi(\xi,\upsilon,\zeta) = (\xi,\upsilon,\zeta) + \int_0^{\tau} (\dot{\xi},\dot{\upsilon},\dot{\zeta}) \ dt.
\end{align*}
We set the timestep $\tau = 0.01$ and initial condition $(\xi_0,\upsilon_0,\zeta_0) = (0, 1.0, 1.05)$. For these initial conditions and the parameter values as in \ref{thm:lorenz}, we computed a trajectory for a 40 time units (i.e. 4000 timesteps), illustrated in Figure~\ref{Lorenz_fig}.

\begin{figure}
  \centering
    \includegraphics[width=0.5\textwidth]{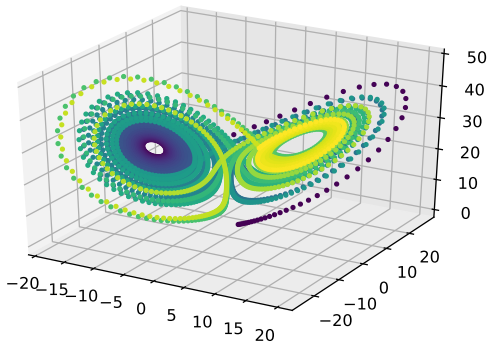}
    \caption{A typical trajectory of the Lorenz system~(\ref{eqn:lorenz}) computed for 4000 timesteps, represented by the individual dots at time intervals $\tau=0.01$. Colour indicates the direction of travel along the trajectory: darkest colours (blue) at the earliest times and lighest colours at the most recent times (yellow).}
    \label{Lorenz_fig}
\end{figure}

We select observation and target functions to be the first and third components of the Lorenz system, i.e. we choose the function $\omega(\xi,\upsilon,\zeta) = \xi$ so that the observations $z_k$ are the $\xi$ components of the trajectory at the sampled time points $t=k\tau$, so that
\begin{align*}
    z_k = \omega\circ\phi^k(\xi_0,\upsilon_0,\zeta_0)).
\end{align*}
We select the target function to be $\omega(\xi,\upsilon,\zeta) = \zeta$ so the targets $u_k$ are the $\zeta$ components of the trajectory:
\begin{align*}
    u_k = u\circ\phi^k(\xi_0,\upsilon_0,\zeta_0)).
\end{align*}
The trajectories of these two components of observations and targets are shown in Figure \ref{xz_fig}(a) and (b), respectively. 

\begin{figure}
    \centering
    \begin{subfigure}[b]{0.5\textwidth}
        \includegraphics[width=\textwidth]{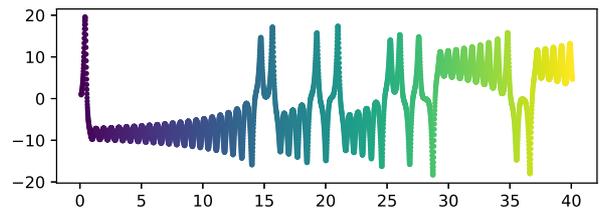}
        \caption{The $\xi$-component of the Lorenz trajectory (vertical axis) plotted against time (horizontal axis).}
        \label{fig:gull}
    \end{subfigure}
    ~ 
    \begin{subfigure}[b]{0.5\textwidth}
        \includegraphics[width=\textwidth]{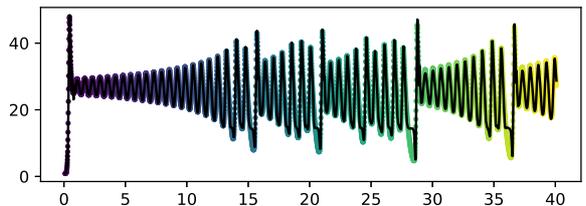}
        \caption{The $\zeta$-component of the Lorenz trajectory (vertical axis) plotted against time (horizontal axis). The black line at the $k^\mathrm{th}$ timestep indicates the approximation to this target time series given by $W^{\top}_{\infty}x_k$.}
        \label{fig:tiger}
    \end{subfigure}
    \caption{Observations $z_k$ and targets $u_k$ drawn from the Lorenz trajectory.}
    \label{xz_fig}
\end{figure}

Our goal is to use an ESN to predict the targets based on the observations. So, we set up an ESN with the following parameters:
\begin{itemize}
    \item Reservoir size: $T = 300$,
    \item Activation function: $\sigma = \tanh$,
    \item Input matrix $C$ and bias vector $\zeta$: i.i.d uniform random variables $\sim U[-0.05,0.05]$,
    \item Reservoir matrix $A$: i.i.d uniform random variables rescaled so that $\lVert A \rVert_2 = 1$,
    \item Regularisation parameter $\lambda = 10^{-9}$.
\end{itemize}
Iterating the ESN with observations $z_k$ creates a discrete-time sequence of reservoir states $x_k$, illustrated in Figure \ref{Reservoir_Lorenz_fig}, which shows a projection of the reservoir states onto their first the principal components.
\begin{figure}
  \centering
    \includegraphics[width=0.5\textwidth]{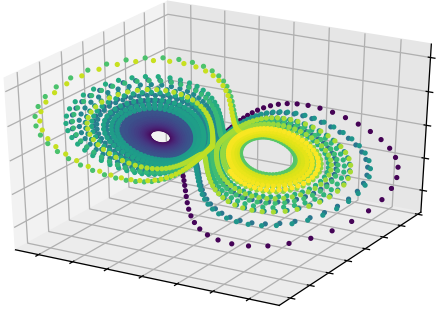}
    \caption{Illustration of the reservoir states of the ESN driven by inputs $z_k$ being the discrete-time samples observed from a trajectory of the Lorenz system. The figure shows the projection of the reservoir states onto their first 3 principal components.}
    \label{Reservoir_Lorenz_fig}
\end{figure}
We then solved the least squares problem
\begin{align*}
    \min_{W}\sum_{k=0}^{\ell-1} \lVert W^{\top} x_{k} - u_k \rVert^2 + \lambda\lVert W \rVert^2 \nonumber
\end{align*}
to determine the output layer $W$ using the SVD. This offline learning method is described by \cite{Deblurring_2006}. Our aim here is to understand how increasing the number of data points $\ell$ improves our approximation of the target function $u$. So we repeated this process with \emph{fewer} observation-target pairs, from 300 in increments of 100 up to 4000. For each value of $\ell$, we compute the best-fit readout layer $W$. We repeated this process once more for a 20,000 time step (i.e. 200 time unit) trajectory and computed the readout layer which for this case we denote by $W_{\infty}$, assuming that it is extremely close to the readout layer we would obtain in the limit of infinitely many time steps. For each readout layer $W$ obtained using fewer data points ($300 \leq \ell \leq 4000$) we estimated the error on the readout layer which we denote by WE:
\begin{align*}
    \text{WE} = \frac{\rVert W - W_\infty \lVert}{\lVert W_{\infty} \rVert},
\end{align*}
and the root mean square error (RMSE) between the targets and the approximation \emph{for the entire 20,000 point trajectory} 
\begin{align*}
    \text{RMSE} = \sqrt{\frac{1}{20000}\sum_{k=0}^{20000-1}\rVert W^{\top}x_k - u_k \lVert^2}.
\end{align*}
We expect that as the the number of data points $\ell$ grows the WE and RMSE will converge. The central limit theorem suggests that the matrix vector pairs $(\Sigma_{\ell},v_{\ell})_{_{\ell \in \mathbb{N}}}$ (which satisfy the Gauss normal equations $\Sigma_{\ell} W_{\ell} = v_{\ell}$)
will converge in law to a multivariate normal distribution, with standard deviation converging with order $1/\sqrt{\ell}$; as the number of data points $\ell$ tends to infinity. This suggests (but does not strictly imply) that the WE and RMSE might converge at a similar rate. We have been unable to derive expressions for the convergence of the RMSE and WE and remark that the need to compute $W$ via a least-squares fit means it is not obvious that these would share the convergence rate of $\Sigma_\ell$ and $v_\ell$. Typical numerical results for the convergence of the RMSE and WE are illustrated in Figure \ref{error_fig}.

\begin{figure}
    \centering
    \begin{subfigure}[b]{0.5\textwidth}
        \includegraphics[width=\textwidth]{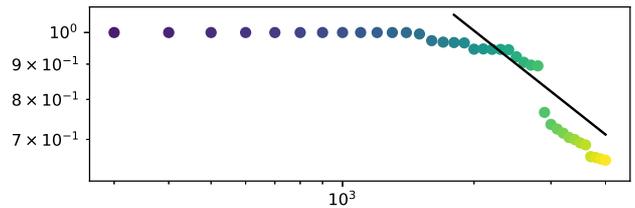}
        \caption{Log--log plot of the error on the linear readout layer $W$ (vertical axis) against number of data points (horizontal axis) used to train the readout layer $W$. The line $y = 45/\sqrt{\ell}$ is plotted in black as a guide to the eye.}
        \label{fig:W_error}
    \end{subfigure}
    ~ 
    \begin{subfigure}[b]{0.5\textwidth}
        \includegraphics[width=\textwidth]{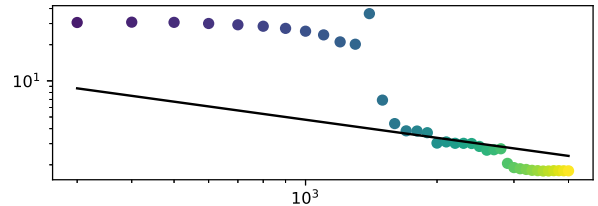}
        \caption{Log--log plot of the root mean square error (RMSE) (vertical axis) against number of data points (horizontal axis) used to train the readout layer $W$. The line $y = 150/\sqrt{\ell}$ is plotted in black as a guide to the eye.}
        \label{fig:RMS_error}
    \end{subfigure}
    \caption{Convergence of the error on the readout layer (WE) and convergence of the root mean square error (RMSE) displayed in log--log plots. Black lines indicate convergence with order $1/\sqrt{\ell}$ and are shown in order to compare the convergence to what might be expected if a central limit theorem applied.}
    \label{error_fig}
\end{figure}

The figures reveal that the convergence of the RMSE and WE is complicated. We observe sudden jumps which appear when the Lorenz trajectory switches to a different wing in the attractor, at such times presumably the ESN rapidly acquires new independent information which improves the fit. Furthermore, the convergence at least over this range of trajectory lengths does not (convincingly) converge with order $1/\sqrt{\ell}$. Since the sudden jumps occur on a timescale intrinsic to the dynamical system $\phi$ we conclue that the internal structure of the attractor and its dynamics plays an important role in the evolution of the error; appealing to the asymptotic behavior may not always be useful.

We pushed the numerics further, hoping to detect an asymptotic regime by repeating the numerical experiments with a much longer trajectory. We computed $W_{\infty}$ for a $100 \ 000$ point trajectory and compared this to the $W$ obtained for shorter time series of lengths $\ell = 1000, 2000, \ldots, 98000$. For each $\ell$ we computed the WE with 10 randomly generated realisations of the ESN. The results are shown in Figure \ref{jon_data_fig} and are also (sadly) inconclusive; there is no obvious regime over which the error decreases as a power law. Sudden decreases as the trajectory switches lobes on the attractor are still visible, and the rate of convergence remains complicated.

\begin{figure}
  \centering
    \includegraphics[width=0.5\textwidth]{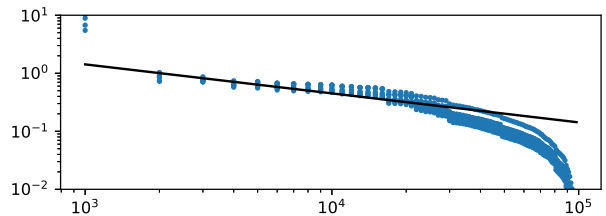}
    \caption{The error on the readout layer (WE) (vertical axis) shown against the number of data points $\ell$ (horizontal axis). The black line has equation $y = 45/\sqrt{\ell}$ as a guide to the eye. Results for 10 separate realisations of an ESN are shown.}
    \label{jon_data_fig}
\end{figure}

\section{Conclusions and future work}
\label{conclusions}

The main result of this paper (Theorem \ref{least_sqs_thm}) states that an ESN trained on a sequence of observations from an ergodic dynamical system (with invariant measure $\mu$) using Tikhonov least squares will $L^2(\mu)$ approximate any target function $u$. We then summarised the result by \cite{Luzzatto2005} which implies the Lorenz attractor exists and is mixing, hence ergodic. This allowed us to conclude that an ESN trained on a sequence of scalar observations taken from the Lorenz system using Tikhonov least squares should $L^2(\mu)$ approximate the dynamics in the attractor.
In section \ref{Numerical_section} we simulated the Lorenz system ourselves and designated the $\xi$ and $\zeta$ components observations and targets respectively. We confirmed that as the number of data points $(\xi_k,\zeta_k)$ grew, the approximation of the target function improved. A good approximation was reached before the number of data points was large enough for the central limit theorem to (perhaps) become relevant. This suggests that (perhaps unfortunately) this asymptotic result may have limited practical use.

We discussed in section~\ref{Training_Thm_for_ESNs} that the $L^2(\mu)$ norm is weaker than the $C^1$ norm, in the sense that convergence in $C^1$ implies convergence in $L^2(\mu)$, while the converse does not hold. This is somewhat unsatisfying, because (topologically conjugate) time series forecasting requires the autonomous phase of the ESN to be a $C^1$ approximator of the embedded (structurally stable) dynamics.

It may be a fruitful to develop a training method beyond Tikhonov least squares that guarantees a $C^1$ approximation. Alternatively, it may be intriguing to explore under what conditions Tikhonov least squares does provide a sufficiently good $C^1$ approximation, which appears to happen frequently in simulations. Authors including \cite{Pathak2017}, \cite{Vlac_2019}, and \cite{Embedding_and_approximation_theorems} have demonstrated that an ESNs trained with Tikhonov least squares can replicate topological invariants of dynamical systems like Lyapunov exponents, fixed point eigenvalues, and homology groups, suggesting a sufficiently good $C^1$ approximation was achieved.

Though the $L^2(\mu)$ approximation may not be sufficient for topological results, it may be powerful enough to prove interesting results about ESNs applied to control problems. We can view a control system as a dynamical system, for which we have at every state $x \in M$ a set of actions $a \in \mathcal{A}$ available to us. Then we seek a map $\pi : M \to \mathcal{A}$, called an optimal controller (in control theory), or an optimal policy (in reinforcement learning), which maximises some reward function. To determine the value of a policy $\pi$ it suffices to determine the value function $u : M \to \mathbb{R}$ which, we can in principal approximate with an ESN from only partial observations of the control system. Developing algorithms to find the optimal controller/policy may be a rewarding direction of future work.

We also believe much of the theory presented here could be generalised or modified for other recurrent neural networks such as long short term memory networks (LSTMs). LSTMs are used extensively in industry and perform very well at \emph{context dependant} time series problems. These are problems where events that happened a long time in the past may suddenly become important in the present. The ESN is not well suited to such problems, because the importance of events necessarily decays (at least) exponentially quickly as we move further into the past, while the structure of an LSTM sidesteps this problem. A detailed explanation of the architecture is provided by \cite{iet:/content/conferences/10.1049/cp_19991218}. Equations for a peephole LSTMS are listed below
\begin{align}
    f_k &= \varphi_g(A_f c_k + W^{\text{in}}_f u_k) \nn \\
    i_k &= \varphi_g(A_i c_k + W^{\text{in}}_i u_k) \nn \\
    o_k &= \varphi_g(A_o c_k + W^{\text{in}}_o u_k) \nn \\
    c_k &= f_k \odot c_{k-1} + i_k \odot  \varphi_c(W^{\text{in}}_c u_k) \nn \\
    h_k &= \varphi_h(o_k \odot c_k) \nn
\end{align}
where $f_k, i_k, o_k, c_k, h_k \in \mathbb{R}^n$ are the vectors of the forget gate, input gate, output gate, cell state, and hidden state (also known as the output state) associated to the LSTM at time $k$. Next, $u_k \in \mathbb{R}$ is the scalar input of the LSTM at time $k$ and $\varphi_g : \mathbb{R}^n \to \mathbb{R}^n$ is a componentwise sigmoid function, $\varphi_c : \mathbb{R}^n \to \mathbb{R}^n$ is the componentwise $\tanh$ function, and $\varphi_h : \mathbb{R}^n \to \mathbb{R}^n$ is some function that is usually the identity map. $A_f, A_i, A_o, A_c$ are $n \times n$ matrices and $W^{\text{in}}_f,W^{\text{in}}_i,W^{\text{in}}_o,W^{\text{in}}_c$ are $1 \times n$ matrices. Finally, the symbol $\odot$ here represents the Hadamard product (taking componentwise product of 2 vectors).

We can see that LSTMs admit ESNs as a special case by fixing $A_f = 0$, $W^{\text{in}}_f = 0$, $b_f = 0$, $b_i = 0$, $b_c = \text{arc}\tanh(1/2)$, $W^{\text{in}}_c = 0$. It may therefore interest the academic community studying LSTMs, as well as those with industrial applications in mind, to generalise the theory of ESNs presented here and elsewhere to LSTMs. 

One shortcoming of Echo State Networks (that is typical for a machine learning paradigm) is that physical information
 about the underlying dynamical system is typically ignored. The question of how one might integrate some basic knowledge of the underlying dynamical system into the ESN architecture was recently explored numerically by \cite{arXiv:2001.04027} and \cite{arXiv:2001.02982}. Developing their ideas further may be an intriguing direction of future work.
 
\section*{Acknowledgements}

We thank the examiners of A.G. Hart's PhD confirmation viva, Alastair Spence and Chris Guiver, who offered helpful criticism of much of the material which formed the basis of this paper. We are grateful to an anonymous reviewer for their comments and careful reading of the manuscript which have significantly helped to improve its presentation.
A.G. Hart is supported by a scholarship from the EPSRC Centre for Doctoral Training in Statistical Applied Mathematics at Bath (SAMBa), under the project EP/L015684/1.

\bibliography{mybibfile}

\end{document}